\documentclass[conference,letterpaper]{IEEEtran}

\renewcommand{\baselinestretch}{0.96}
\usepackage[utf8]{inputenc} 
\usepackage[T1]{fontenc}    
\usepackage{hyperref}       
\usepackage{url}            
\usepackage{booktabs}       
\usepackage{amsfonts}       
\usepackage{nicefrac}       
\usepackage{microtype}      

\usepackage{comment}
\usepackage{lmodern}
\usepackage{scalefnt}
\usepackage{setspace}
\usepackage{romannum}
\usepackage[usenames,dvipsnames]{xcolor}
\usepackage{tkz-berge}
\usetikzlibrary{fit,shapes}
\usepackage{calrsfs}
\DeclareMathAlphabet{\pazocal}{OMS}{zplm}{m}{n}
\usepackage{graphicx}
\usepackage{graphics}
\usepackage{epsfig}
\usepackage{mathptmx}
\usepackage{caption}
\usepackage{subcaption}
\usepackage{tikz}
\usetikzlibrary{positioning}
\usepackage{amsmath, amssymb}
\usetikzlibrary{arrows}
\usepackage{dsfont}

\usepackage{graphicx}
\usepackage{subcaption}

\usepackage{kantlipsum}
\usepackage{amsthm}
\usepackage{amsfonts}
\usepackage{setspace}
\usepackage{multirow}
\usepackage[english]{babel}
\usepackage{float}
\usepackage{algorithmicx}
\usepackage[section]{algorithm}
\usepackage{algpascal}
\usepackage{algc}
\usepackage{algcompatible}
\usepackage{multicol}
\usepackage{algpseudocode}
\let\oldReturn\Return
\renewcommand{\Return}{\State\oldReturn}
\usepackage{mathtools}
\usepackage{mathptmx}
\usepackage{mathtools, cuted}
\usepackage{textcomp}
\usepackage[english]{babel}
\usepackage{rotating}
\usepackage{pgfplots}
\pgfplotsset{compat=1.5}
\usepackage{siunitx}
\usepackage{pgfplotstable}
\usepackage{filecontents}

\usepackage{bbold}
\usepackage{float}

\newtheorem{theorem}{Theorem}
\newtheorem{corollary}[theorem]{Corollary}
\newtheorem{lemma}[theorem]{Lemma}

\newtheorem{assume}{Assumption}

\DeclareMathAlphabet{\pazocal}{OMS}{zplm}{m}{n}

\newcommand{\bR}{\mathbb{R}}
\newcommand{\A}{{\pazocal{A}}}
\newcommand{\C}{{\pazocal{C}}}
\newcommand{\R}{\pazocal{R}}
\newcommand{\mP}{{\mathbb{ P}}}
\renewcommand{\P}{{\pazocal{ P}}}
\renewcommand{\l}{\ell}
\newcommand{\T}{{\pazocal{T}}}
\newcommand{\E}{{\pazocal{E}}}
\newcommand{\F}{\pazocal{F}}
\newcommand{\fp}{f_p}
\newcommand{\yt}{y_{a, i}}
\newcommand\scalemath[2]{\scalebox{#1}{\mbox{\ensuremath{\displaystyle #2}}}}
\begin{document}

\title{Federated  Learning for Heterogeneous Bandits with Unobserved  Contexts} 

\author{%
  \IEEEauthorblockN{Jiabin Lin and Shana Moothedath}
  \IEEEauthorblockA{Electrical and Computer Engineering\\
                    Iowa State University, Ames, IA, USA\\
                    Email: jiabin@iastate.edu, mshana@iastate.edu}
}

\maketitle


\begin{abstract}
We study the federated stochastic multi-arm contextual bandits with unknown contexts, in which $M$ agents  face $M$  different bandit problems and  collaborate to learn. The communication model consists of a central server and $M$ agents,  and the agents share their estimates with the central server periodically  to learn to choose optimal actions  to minimize the total regret.  We assume that the exact contexts are not observable, and the  agents observe only a distribution of  contexts. Such a situation arises, for instance, when the context itself is a noisy measurement or based on a prediction mechanism. Our goal is to develop a distributed and federated algorithm that facilitates collaborative learning among the agents to select a sequence of optimal actions to maximize the cumulative reward. By performing a  feature vector transformation, we propose an elimination-based algorithm and prove the regret bound for linearly parametrized reward functions. 
Finally, we validated the performance of our algorithm and compared it with another baseline approach using numerical simulations on synthetic data and  the real-world movielens dataset.
\end{abstract}
\section{Introduction}\label{sec:intro}
Online learning for sequential decision-making frequently arises in diverse domains, spanning control systems and robotics \cite{cheung2013autonomous, srivastava2014surveillance}, clinical trials \cite{aziz2021multi}, communications \cite{anandkumar2011distributed}, and recommender systems \cite{li2010contextual}.
The multi-arm bandits (MAB) model the iterative interaction of a system with the environment to learn optimal decisions \cite{lattimore2020bandit}. In MABs, the learner aims to maximize the reward by engaging in rounds where they interact with the environment, choose actions based on current estimates, receive rewards, and update estimates accordingly.
    
    A central challenge in bandit problems is the exploration-exploitation trade-off. While accurate estimation with minimal exploration is ideal for maximizing rewards, insufficient exploration often leads to suboptimal decisions. Striking a suitable exploration-exploitation balance is crucial and this is even more challenging in scenarios where the exact contexts are unknown. In many practical instances, the {\em exact} contexts are often unknown,  where the contexts are noisy measurements or based on certain prediction mechanisms, such as weather or stock market prediction. In such situations, decisions must be made based on limited information.

Over the last decade,  multi-agent MAB problems have been widely studied in many papers \cite{wang2019distributed, lin2022distributed_1, lin2022distributed_2, landgren2021distributed, martinez2019decentralized, hillel2013distributed, korda2016distributed, tao2019collaborative, zhu2021decentralized, sankararaman2019social, wang2022multi}. In multi-agent MAB, multiple agents/learners collaborate to learn and optimize their decisions based on their local knowledge about the environment. Recently, {\em federated} learning has gained attention due to increasing focus on security and privacy.  In federated learning, the agents do not share the raw data with the other agents or the server but rather share only the local estimates.  Federated MABs are also studied in some of the recent papers \cite{zhu2021federated, shi2021federated1, shi2021federated2,agarwal2020federated}.

Our goal in this paper is to present an algorithm for the distributed and federated contextual multi-armed bandit problem with $M$ agents/learners when the exact context is unknown. In our setting, $M$ agents work collaboratively and concurrently to learn optimal actions to maximize the total (cumulative) reward. The agents communicate with a central server in a periodic fashion and update their local models by utilizing the global information from the central server. 
Bandit learning with hidden contexts is challenging for the learner because estimating the reward function relies on noisy observations due to unknown exact contexts. To overcome these challenges, we make use of a feature mapping technique used in \cite{kirschner2019stochastic} for a single-agent bandit setting to transform the problem. After this transformation, a new set of feature vectors is presented so that under this set of $d$-dimensional context feature vectors, the reward is an unbiased observation for the action choice. Motivated by the approach in \cite{huang2021federated}, we propose a federated algorithm and prove the regret bound for the linear reward function. The key difference in our approach is that we consider a setting where the contexts are unobservable, while in \cite{huang2021federated}, the contexts are known.  

This paper makes the following contributions.
\begin{enumerate}
\item[$\bullet$] We model a distributed $M$-agent, federated stochastic linear contextual bandit where the agents face different bandits and the exact contexts are unknown.
\item[$\bullet$] We present a Fed-PECD algorithm and prove the regret bound for distributed stochastic bandits with linear parametrized reward function when the exact context is hidden and the context distribution is available. 
\item[$\bullet$] We validated the performance of our approach and compared the different models via numerical simulations on synthetic data and on real-world movielens data. 
\end{enumerate}

The organization of the paper is as follows. In Section~\ref{sec:rel}, we present the related work. In Section~\ref{sec:problem}, we discuss the notations used in the paper,  problem formulation, and a motivating scenario for the problem. In Section~\ref{sec:alg}, we present our proposed algorithm and regret analysis of the algorithm. In Section~\ref{sec:simulation}, we present the numerical experiments and in  Section~\ref{sec:con}, we give the concluding remarks.

\section{Related Work}\label{sec:rel}
Distributed bandits have been widely studied for different learning models and communication models.   
 The stochastic multi-agent linear contextual bandit problem was studied in \cite{wang2019distributed}, and an Upper Confidence Bound (UCB)-based algorithm was proposed with guarantees. The communication model considered in \cite{wang2019distributed} consists of a central server that can communicate with all the agents. In \cite{wang2019distributed}, the agents observe the exact contexts while we consider cases where the contexts are unobservable and only the context distribution is available. 
%
 Contextual bandits with context distribution (unknown context) were studied in \cite{kirschner2019stochastic} for a single-agent setting. Later, references \cite{lin2022distributed_1, lin2022distributed_2} considered a distributed contextual MAB problem when the contexts are hidden and proposed an elimination-based algorithm. We note that the learning models in \cite{wang2019distributed, lin2022distributed_1, lin2022distributed_2} considered a setting where all agents encounter the same bandit. In contrast, this paper introduces a context where agents face different bandit problems, thus characterizing {\em heterogeneous agents}.
 In \cite{Jiabin_Shana_ACC}, we studied the feature selection problem for multi-agent bandits with a fixed action set, which is different from this work, which studies CBs.

Decentralized MAB problems have been studied in many papers including \cite{landgren2021distributed, martinez2019decentralized, hillel2013distributed, korda2016distributed, tao2019collaborative, zhu2021decentralized, sankararaman2019social, wang2022multi} and different communications models were considered, such as each agent communicates with only two other agents \cite{korda2016distributed},  each agent communicate only once \cite{hillel2013distributed}, and each agent can communicate with every other agent. We note that the model considered in this paper is a centralized and federated communication model.
%
Federated decentralized MAB problems are studied in  \cite{li2020federated, dubey2020differentially} with the primary focus on differential privacy.
%
In \cite{zhu2021federated} a fully decentralized federated MAB problem was investigated where agents find the optimal arm with limited communication and can only share their rewards with neighbors. Reference \cite{shi2021federated1, shi2021federated2} considered a federated MAB problem where the local agent and central server collaborate to maximize the cumulative reward based on the local and global reward with and without personalization, respectively. A federated residual learning algorithm was introduced in \cite{agarwal2020federated} for the federated bandit problem where the agent chooses the optimal arm based on the global and local model with communication latency.

\section{Notations and Problem Formulation}\label{sec:problem}
\subsection{Notations}
The norm of a vector $z \in \bR^d$ with respect to a matrix $V \in \bR^{d \times d}$ is defined as $\|z \|_{V}: = \sqrt{z^\top V z}$. $|z|$ for a vector $ z$ denotes element-wise absolute values. Further, ${}^{\top}$ denotes matrix or vector transpose, $A^{\dagger}$ denotes the pseudo-inverse of a matrix $A$, and $\langle \cdot, \cdot \rangle$ denotes inner product. For an integer $N$, we define $\left[ N \right]:=\{1,2,\ldots,N\}$. We use $\| \cdot\|$ to denote the induced $\l_2$ norm.

\subsection{Problem Formulation}
In this section, we first present the standard linear contextual bandit problem. Let $\A$ be the action set with $|\A|=K$, $\C$ be the context set, and the environment is defined by a fixed and unknown reward function $y: \A \times \C \rightarrow \mathbb{R}$. In a linear bandit setting, at any round $t \in \mathbb{N}$, the agent observes a context $c_t \in \C$ and chooses an action $a_t \in \A$. Each context-action pair $\left( a,c\right)$, $a \in \A$ and $c \in \C$, is associated with a feature vector $\phi_{a,c} \in  \mathbb{R}^d$, i.e., $\phi_{a_t, c_t} = \phi\left( a_t, c_t\right)$. Upon selection of an action $a_t$, the agent observes a  reward $y_t \in  \mathbb{R}$, $y_t :=  \langle\theta^\star, \phi_{a_t, c_t}  \rangle + \eta_t,$
where $\theta^\star \in \mathbb{R}^d$ is the unknown reward parameter, $ \langle\theta^\star, \phi_{a_t, c_t}  \rangle  = r\left( a_t, c_t\right)$ is the expected reward for action $a_t$ at round $t$, i.e., $r\left( a_t, c_t\right) = \mathbb{E}[y_t]$, and $\eta_t$ is $\sigma$-Gaussian, additive  noise. The goal is to choose optimal actions $a_t^{\star}$ for all $t \in T$ such that  the cumulative reward, $\sum_{t=1}^T y_t$, is maximized. This is equivalent to minimizing the cumulative (pseudo)-regret denoted as 
\begin{equation}
\R_T = \sum_{t=1}^T\langle\theta^\star, \phi_{a_t^{\star}, c_t}^t  \rangle - \sum_{t=1}^T\langle\theta^\star, \phi_{a_t, c_t}^t  \rangle.\label{eq:regret}
\end{equation}
Here $a_t^{\star}$ is the optimal/best action for context $c_t$, and $a_t$ is the action chosen by the agent for context $c_t$.

In this paper, we consider a federated linear contextual bandit problem consisting of $M$ agents with {\em heterogeneous data}. 
The agent's goal is to learn collaboratively and concurrently.
 Each agent $i\in [M]$ is a user in the system whose user profile is denoted by $c_i \in \mathbb{R}^d$ and all agents have the same action set $\A=[K]$.
The context vector $c_i$ cannot be observed by the agent $i$ and only the context distribution $\mu_i$ is available.
This situation occurs in various scenarios, such as in recommender systems, where other family members may use the same user profile to log in to Netflix. Consequently, the learner lacks knowledge of the exact user information (context vector).

The goal of the agents is to learn to choose optimal actions based on their local observation. To achieve this, at every time step, the agents choose actions from their action set and receive the reward $\yt =\langle\theta_{a_{i, t}}, \phi_{a_{i, t}, c_i} \rangle+ \eta_{i, t}$, where $\phi_{a_{i, t}, c_i}$ is the feature vector of agent $i$ for action $a_{i,t}$ and  $\eta_{i, t}$ is the random noise. 
Note that the rewards received by the agents for the same action are different as it is a function of the agent's context vector, thus capturing the heterogeneous nature of the agents. Further, for action $a \in \A$, the reward parameter $\theta_a$ depends on the action and is the same across agents. Such a model hence captures the heterogeneous nature of the agents while still allowing a collaborative learning paradigm due to the shared parameter $\{\theta_a \}_{a \in \A}$.  The linear structure of the reward function captures the intrinsic correlation between rewards for different agents pulling the same arm with the same parameter $\theta_a$.

Our goal is to learn an optimal mapping $\mu_i \rightarrow a_i^{\star}$ of context distribution to action that maximizes the cumulative reward,  where $a_i^{\star}$ is the optimal action for user/agent $i$,  for all $i\in [M]$. In other words, we try to find the optimal estimated action that minimizes the cumulative regret
\begin{align}
\R\left( T\right) =  \sum_{i=1}^M \sum_{t=1}^T\langle\theta_{a_i^\star}, \phi_{a_i^{\star}, c_i} \rangle - \sum_{i=1}^M  \sum_{t=1}^T\langle\theta_{a_{i, t}}, \phi_{a_{i, t}, c_i} \rangle,\label{eq:regret_M}
\end{align}
where $a_i^\star \in \A$ is the optimal action for agent $i$ provided we know $\phi$ and $\mu_i$, but not $c_i$. 

Our communication model consists of a central server in the system which can communicate with each agent periodically with zero latency. Each agent  shares its local estimate with the central server, the server then aggregates the estimates and broadcasts a  global estimate to each agent. Each agent now updates their local model. Note that communication is always a major bottleneck, and we need to carefully consider its usage to keep the communication cost as small as possible. 
%
The agents can only share their estimates with the central server to obtain the global estimated model, without sharing raw data. In this way, the local data is always private for each agent and hence the communication model is federated. We have the following standard assumptions on the parameters.

\begin{assume}
There exist constants $s \geqslant 0$, $0 \leqslant \l, L \leqslant 1$ such that $\left\|\theta_{a}\right\|_2 \leqslant s$, $0 \leqslant \l \leqslant \left\|\phi_{a, c_i}\right\|_2 \leqslant L \leqslant 1$, for all $i \in [M]$ action $a \in \A$. 
\end{assume}
\begin{assume}\label{assume:noise}
Each element $\eta_{i, t}$ of the noise sequence $\{\eta_{i, t}\}_{i=1, t=1}^{M, \infty}$ is a 1-subgaussian sampled independent with $\mathbb{E}[\eta_{i, t}] = 0$, $\mathbb{E}[e^{\lambda \eta_{i, t}}] \leqslant e^{\frac{\lambda^2}{2}}$ for any $\lambda >0$. 
\end{assume}

\section{The Proposed Algorithm and Guarantee}\label{sec:alg}
This section presents our proposed algorithm and its regret guarantee.
\begin{algorithm}[t]
\caption{Fed-PECD: agent $i$}\label{alg: client}
\begin{algorithmic}
\State \textit {\bf Input: $T, M, K$, $\alpha$, $\fp$} 
\end{algorithmic}
\begin{algorithmic}[1] 
\State Nature chooses $\mu_t \in \P\left( \C\right)$ and learner observes $\mu_t$
\State Set $\Psi_t = \{{\psi}_{a, \mu_i}: a \in \A\}$ where $\{{ \psi}_{a, \mu_i} := \mathbb{E}_{c_i \sim \mu_i}[\phi_{a, c_i}]\}$\label{step:psi}
\State  \textit {\bf  Initialization:} Pull each arm $a \in \A$ and receive reward $\yt; \hat{\theta}_{a, i}^{0} \leftarrow \frac{\yt \psi_{a, \mu_i}}{\left\| \psi_{a, \mu_i} \right\|^{2}}$; Send $\left\{ \hat{\theta}_{a, i}^{0} \right\}_{a \in \A}$ to the server; $\A_{i}^{0} \leftarrow \A; p \leftarrow 1$

\While {not reaching the time horizon $T$}
\State Receive $\left\{ \left( \hat{\theta}_{a}^{p}, V_{a}^{p}\right) \right\}_{a \in \A^{p-1}}$ from the server
\For {$a \in \A_{i}^{p-1}$}
\State $\hat{r}_{a, i}^{p} \leftarrow \psi_{a, \mu_i}^{\top} \hat{\theta}_{a}^{p}, \quad u_{a, i}^{p} \leftarrow \alpha \left\| \psi_{a, \mu_i} \right\|_{V_{a}^{p}} / \l$
\EndFor
\State $\scalemath{0.9}{\hat{a}_{i}^{p} \leftarrow \arg \max\limits_{a \in \A_{i}^{p-1}} \hat{r}_{a, i}^{p}, \A_{i}^{p} \leftarrow \{a \in \A_{i}^{p-1} \mid \hat{r}_{a, i}^{p} + u_{a, i}^{p} \geqslant \hat{r}_{\hat{a}_{i}^{p}, i}^{p} - u_{\hat{a}_{i}^{p}, i}^{p}\}}$\label{line:action-set}
\State Send $\A_{i}^{p}$ to the central server
\State Receive $f_{a, i}^{p}$ for all $a \in \A_{i}^{p}$
\For {$a \in \A_{i}^{p}$}
\State $\quad$ Pull arm $a$ $f_{a, i}^{p}$ times and receive $\left\{\yt\right\}_{t \in \T_{a, i}^{p}}$
\EndFor
\State Send $\left\{ \hat{\theta}_{a, i}^{p} \right\}_{a \in \A_{i}^{p}}$ to the server; pull $ \hat{a}_{i}^{p} $ until phase length equals $ \fp + K $
\State $ p \leftarrow p + 1 $
\EndWhile
\end{algorithmic}
\end{algorithm}
\subsection{Federated Phased Elimination with Context Distribution (Fed-PECD) Algorithm}
\begin{algorithm}[t]
\caption{Fed-PECD: Central server}\label{alg: server}
\begin{algorithmic}
\State \textit {\bf Input: $T, M, K$, $\alpha$, $\fp$} 
\end{algorithmic}
\begin{algorithmic}[1] 
\State  \textit {\bf  Initialization:} Receive $\left\{\hat{\theta}_{a, i}^{0}\right\}_{a, i}; \bar{e}_{a, i} \leftarrow \frac{\hat{\theta}_{a, i}^{0}}{\left\|\hat{\theta}_{a, i}^{0}\right\|}$ for all $i \in[M], a \in\A; V_{a}^{1} \leftarrow ( \sum_{i \in[M]} \frac{\hat{\theta}_{a, i}^{0} ( \hat{\theta}_{a, i}^{0})^{\top}}{\left\|\hat{\theta}_{a, i}^{0}\right\|})^{\dagger}$, $\hat{\theta}_{a}^{1} \leftarrow V_{a}^{1} \left( \sum_{i \in[M]} \hat{\theta}_{a, i}^{0}\right)$ for all $a \in \A;$ Broadcast $\left\{\hat{\theta}_{a}^{1}, V_{a}^{1}\right\}_{a \in \A}; p \leftarrow 1$
\While {not reaching the time horizon $T$}
\State Receive $\scalemath{0.89}{\left\{\A_{i}^{p}\right\}_{i \in[M]}}$; Set $\scalemath{0.89}{\A^{p} \leftarrow \cup_{i=1}^{M} \A_{i}^{p}}, \scalemath{0.89}{\R_{a}^{p} \leftarrow\left\{i: a \in \A_{i}^{p}\right\}}$
\State Solve the multi-agent G-optimal design problem, and obtain solution $\pi^{p} = \left\{\pi_{a, i}^{p}\right\}_{i \in[M], a \in \A_{i}^{p}}$
\State For every agent $i$, send $\left\{f_{a, i}^{p} := \left\lceil \pi_{a, i}^{p} \fp\right\rceil \right\}_{a \in \A_{i}^{p}}$
\State Receive $\left\{( a, \hat{\theta}_{a, i}^{p})\right\}_{a \in \A_{i}^{p}}$ from each agent $i$
\For {$a \in A^{p}$}
\State $\scalemath{0.9}{V_{a}^{p+1} \leftarrow ( \sum_{i \in \R_{a}^{p}} f_{a, i}^{p} \frac{\hat{\theta}_{a, i}^{p} \left( \hat{\theta}_{a, i}^{p}\right)^{\top}}{\left\| \hat{\theta}_{a, i}^{p} \right\|^{2}})^{\dagger}, \hat{\theta}_{a}^{p+1} \leftarrow V_{a}^{p+1} \left( \sum_{i \in \R_{a}^{p}} f_{a, i}^{p} \hat{\theta}_{a, i}^{p}\right)}$
\EndFor
\State Broadcast $\left\{ ( \hat{\theta}_{a}^{p+1}, V_{a}^{p+1}) \right\}_{a \in \A^{p}}$ to all agents
\State $p \leftarrow p+1$
\EndWhile
\end{algorithmic}
\end{algorithm}
The Fed-PECD algorithm consists of two parts: an agent part and a central server part. We present the pseudocode of the algorithm in Algorithm~\ref{alg: client} and Algorithm~\ref{alg: server}. The agents and the server work in phases. The agents communicate with the central server at the end of each phase and update their local models. We denote $\A_i^p$ as the active arm set for agent $i$ in phase $p$ and $\A^p = \cup_{i = 1}^{M} \A_i^p$ as the active arm set for all the agents in phase $p$, where $M$ is the total number of agents. Let $\R_a^p = \{i: a \in \A_i^p \}$ be the set of agents with active arm $a$ in phase $p$, $\T_{a, i}^{p}$ be the number of times we pulled action $a$ for agent $i$ in phase $p$. 
We denote the length of phase $p$ as $\fp + K$, where $K$ is the number of actions.
Now, we elaborate on the process of our algorithm. 

In the initialization phase, the agents explore by pulling the arms $a \in \A$ and receive rewards $\yt$ for the chosen action. Each agent then obtains a local estimate of the reward parameter $\hat{\theta}^0_{a, i}$ and shares it with the central server. The central server aggregates all the local estimates of the agents to obtain a global estimate $\hat{\theta}_a$ and $V_a$. Then the central server broadcasts $(\hat{\theta}_a, V_a)$ pairs to each agent, and the agents, upon receiving this information, will update their local models before the next phase. This completes the initialization phase and the agents proceed to the first phase of the learning algorithm. 

At each phase $p$, after receiving $( \hat{\theta}_a, V_a)$ pairs, each agent first computes the estimated reward $\hat{r}_{a, i}^{p}$ and confidence interval $u_{a, i}^{p}$ for each action $a$. By using the estimate pair $( \hat{r}_{a, i}^{p}, u_{a, i}^{p})$, the agent finds the optimal action and updates the active action set as in Line~\ref{line:action-set} in Algorithm~\ref{alg: client}. In other words, the actions that result in a low reward are eliminated. The agents communicate the current active action set with the central server, which is then used to obtain $\A^p$ and $\R_a^p$. The central server will also compute $f_{a, i}^{p}$ and broadcast to each agent $i$ by solving the multi-agent G-optimal design \cite[Chapter~21]{lattimore2020bandit} to find the distribution $\pi_{a, i}$ for each action $a$ of each agent. 
After receiving $f_{a, i}^{p}$ from the central server, the agents choose the actions from the active action set according to $f_{a, i}^{p}$ and find the average reward. Then each agent updates $\hat{\theta}_{a, i}^{p}$, shares it with the central server, and updates its potential global matrix with this new $\hat{\theta}_{a, i}^{p}$. Note that the agents do not share $\psi_{a, \mu_i}$  or the chosen actions with the central server, thus the local data is private and agents only share their estimates with the central server, and thus our algorithm is {\em federated}. 

We now present the regret and communication bounds for our algorithm. Recall that $\fp+K$ denotes the length of the $p$-th phase, where $K$ denotes the number of actions.
\begin{theorem}\label{Thm}
Consider time horizon $T$ that consists of $H$ phases with $\fp = cn^p$, where $c$ and $n > 1$ are fixed integers and $n^p$ denotes the $pth$-power of $n$. Let 
$$
\alpha = \min\{\sqrt{2 \log \frac{2 M K H}{\delta}}, \sqrt{2 \log \frac{K H}{ \delta} + d \log \left( k e\right)}\}
$$
where $k > 1$ is a number satisfying $kd \geqslant 2 log\left( KH / \delta\right) + d log\left( ke\right)$. Then, with probability (w.p) at least $1 - \delta$ the cumulative regret of our algorithm scales in 
\[
O\left( \frac{L}{\l} \sqrt{d K M T( \log (K( \log T) / \delta) + \min \{ d, \log M\})}\right)
\]
and communication cost scales in $O\left( M d^2 K \log{T}\right)$.
\end{theorem}

Proof of Theorem~\ref{Thm} is presented in Subsection~\ref{sec:regret analysis}. 

\subsection{Regret Analysis}\label{sec:regret analysis}
To bound our algorithm's cumulative regret and communication cost, we extend the approach in  \cite{huang2021federated} to the case where the contexts are hidden and not observable. We show that even when the contexts are unobservable, the proof approach extends and our algorithm achieves the same order-wise bounds as in \cite{huang2021federated}. We present the proof for the sake of completeness. We first present the preliminary lemmas and then prove the main result. 

We first define an event $\E\left( \alpha\right)$ as
$$
\E\left( \alpha\right) := \{\exists p \in [H], i \in [M], a \in \A_{i}^{p-1}, \left|\hat{r}_{a, i}^{p}-r_{a, i}\right| \geqslant u_{a, i}^{p} = \alpha \sigma_{a, i}^{p}\},
$$
where $\scalemath{0.9}{\alpha = \min \left\{\sqrt{2 \log \left( 2 M K H / \delta\right)}, \sqrt{2 \log \left( K H / \delta\right) + d \log \left( k e\right)}\right\} }$, and $\sigma_{a, i}^{p}:= \frac{\sqrt{10}}{\ell} \left\|\psi_{a, \mu_i}\right\|_{V_{a}^{p}}$.

We refer to $\E\left( \alpha\right)$ as a ``bad" event and $\E^{c}\left( \alpha\right)$ as a "good" event. We define $\F_{p} := \left\{\hat{\theta}_{a, i}^{p}\right\}_{p \in [H], i \in [M], a \in \A_{i}^{p}}$, the information available at the end of phase $p$. 

We first show that the good event $\E^c\left( \alpha\right)$ occurs with high probability. Then, we show that if a good event occurs, the best action is not eliminated from the action set during the arm elimination procedure. 
\begin{lemma} \label{LT2}
For our Fed-PECD algorithm, $\mP[\E \left( \alpha\right)] \leq \delta$.
\end{lemma}
For proof of Lemma~\ref{LT2}, we refer to the Appendix (Section~\ref{sec:LT2_proof}).
\begin{lemma} \label{LT3}
If $\E^{c}\left( \alpha\right)$ occurs, we must have $a_{i}^{\star} \in \A_{i}^{p}$, i.e., any optimal arm will never be eliminated.
\end{lemma}
For proof of Lemma~\ref{LT3}, we refer to the Appendix in \cite{huang2021federated}.
We now present the regret bound for the good events below.
\begin{lemma} \label{LT4}
If $\E^{c}\left( \alpha\right)$ occurs, the regret of Fed-PECD in phase $p$ is upper bounded by $\frac{4 \sqrt{10} \alpha L}{\l} \sqrt{d K M} \frac{f^{p} + K}{\sqrt{f^{p-1}}}$. 
\end{lemma}
For proof of Lemma~\ref{LT4} we refer to Section~\ref{app-2} in the Appendix.
Using the above lemmas, we now bound the regret below.

{\em Proof of Theorem~\ref{Thm}:}
Recall the superscript $p$ of $n$ is the exponent. Then, with probability at least $1 - \delta$ and $n > 1$, the total regret over the $H$ phases can be bounded as
\begin{align}
&\scalemath{0.9}{R\left( T\right) = \sum_{p=1}^{H} R_p \leqslant \sum_{p=1}^{H} \frac{4\sqrt{10} \alpha L}{\l} \sqrt{d K M} \frac{\fp + K}{\sqrt{f_{p-1}}}} \label{eq:label6_0} \\
&\scalemath{0.9}{= \sum_{p=1}^{H} \frac{4\sqrt{10} \alpha L}{\l} \sqrt{d K M} \left( \sqrt{c} n^{\frac{p+1}{2}} + K \frac{n^{-\frac{p-1}{2}}}{\sqrt{c}}\right)} \label{eq:label6_1} \\
&\scalemath{0.9}{= \frac{4\sqrt{10} \alpha L}{\l} \sqrt{d K M} [\sqrt{c} \left( \frac{1 - \sqrt{n}^{H + 1}}{1 - \sqrt{n}} - 1\right) \sqrt{n}}\nonumber\\
&\scalemath{0.9}{\hspace*{2 cm} + K \frac{1}{\sqrt{c} \left( \frac{1 - \sqrt{n}^{H + 1}}{1 - \sqrt{n}} - 1\right) \frac{1}{\sqrt{n}}}] }\label{eq:label6_2}\\
&\scalemath{0.9}{= \frac{4\sqrt{10} \alpha L}{\l} \sqrt{d K M} \left( \sqrt{cn} \frac{n^{\frac{H + 1}{2}} - \sqrt{n}}{\sqrt{n} - 1} + \frac{K}{\sqrt{c} \frac{\sqrt{n}^{H} - 1}{\sqrt{n} - 1}}\right)} \nonumber \\
&\scalemath{0.9}{= \frac{4\sqrt{10} \alpha L}{\l} \sqrt{d K M} \left( \sqrt{cn} \frac{n^{\frac{H + 1}{2}} - \sqrt{n}}{\sqrt{n} - 1} + \frac{K}{\sqrt{cn} - \sqrt{c}}\right)} \nonumber \\
&\scalemath{0.9}{= \frac{4\sqrt{10} \alpha L}{\l} \sqrt{d K M} \left( \frac{\sqrt{n} \sqrt{n - 1}}{\sqrt{n} - 1} \sqrt{cn \frac{n^{H} - 1}{n - 1}} + \frac{K}{\sqrt{cn} - \sqrt{c}}\right) \hspace*{-3 mm}}\label{eq:label6_3}\\
&\scalemath{0.9}{= \frac{4\sqrt{10} \alpha L}{\l} \sqrt{d K M} \left( \frac{\sqrt{n^{2} - n}}{\sqrt{n} - 1} \sqrt{T} + \frac{K}{\sqrt{cn} - \sqrt{c}}\right).}\label{eq:label6_4} 
\end{align}
Eq.~\eqref{eq:label6_0} follows from Lemma~\ref{LT4}.
Eq.~\eqref{eq:label6_1} follows from $\fp = cn^{p}$.
Eq.~\eqref{eq:label6_2} follows from $\sum_{n = 0}^{N - 1} r^{n} = \frac{1 - r^{N}}{1 - r}$.
Eq.~\eqref{eq:label6_3} follows from $n^{\frac{H}{2}} - 1 \leqslant \sqrt{n^{H} - 1}$.
Eq.~\eqref{eq:label6_4} follows from $T = \sum_{p-1}^{H} \fp + K H \geqslant \sum_{p-1}^{H} c n^{p} = c n \frac{n^{H} - 1}{n - 1}$.

Since
$
\alpha = O\left( \sqrt{\log \left( K\left( \log T\right) / \delta\right) + \min \{d, \log M\}}\right),
$
the regret scales in
$$
O\left( \frac{L}{\l} \sqrt{d K M( \log( K( \log T) / \delta) + \min \{ d, \log M\}}\left( \sqrt{T} + K\right)\right).
$$
When $K = O\left( \sqrt{T}\right)$, the cumulative regret scales as
$$
O\left( \frac{L}{\l} \sqrt{d K M T( \log ( K( \log T) / \delta) + \min \{ d, \log M\}}\right) .
$$

We define communication cost as the number of scalars communicated between the server and the agents. The communication cost  bound follows from Theorem~1 in \cite{huang2021federated} and is omitted here in the interest of space.
\qed

\section{Numerical Experiments}\label{sec:simulation}
\begin{figure*}[t]
\centering
\subcaptionbox{\small \label{fig:3}}{\includegraphics[width=0.35\textwidth, height=0.25\textwidth]{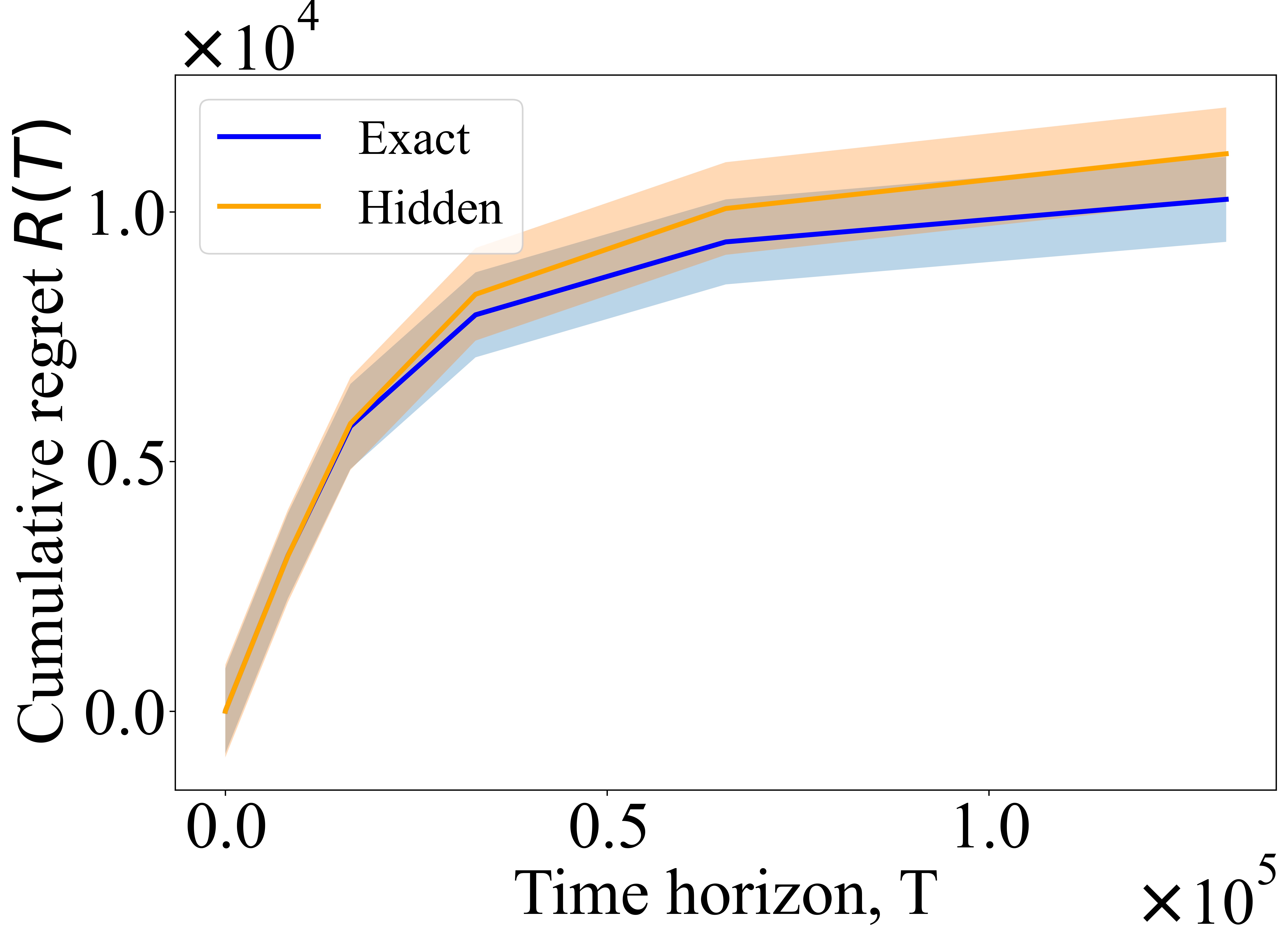}}\hspace{10 mm}
\subcaptionbox{\small \label{fig:4}}{\includegraphics[width=0.35\textwidth, height=0.25\textwidth]{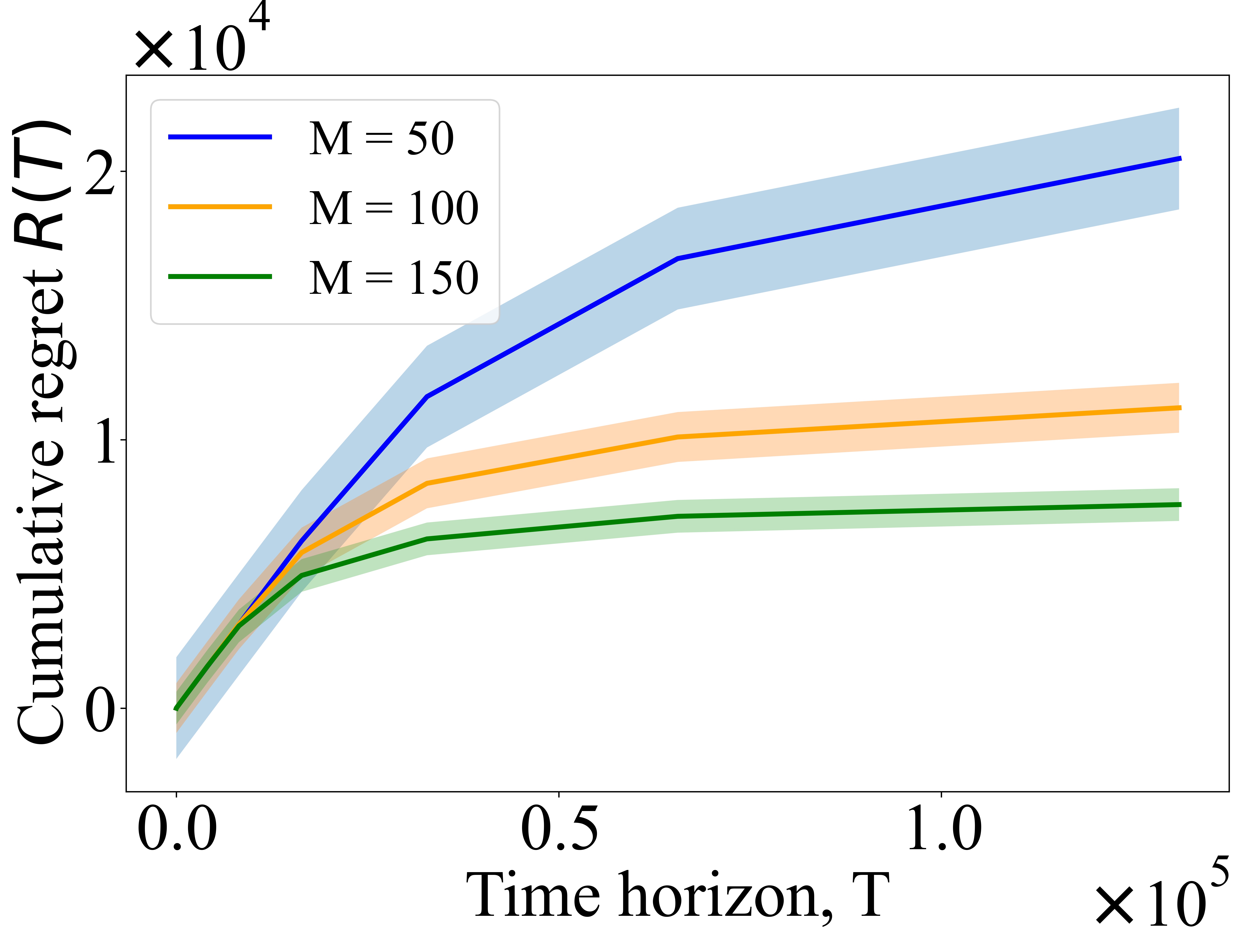}}\hspace{2 mm}

\subcaptionbox{\small \label{fig:5}} {\includegraphics[width=0.35\textwidth, height=0.25\textwidth]{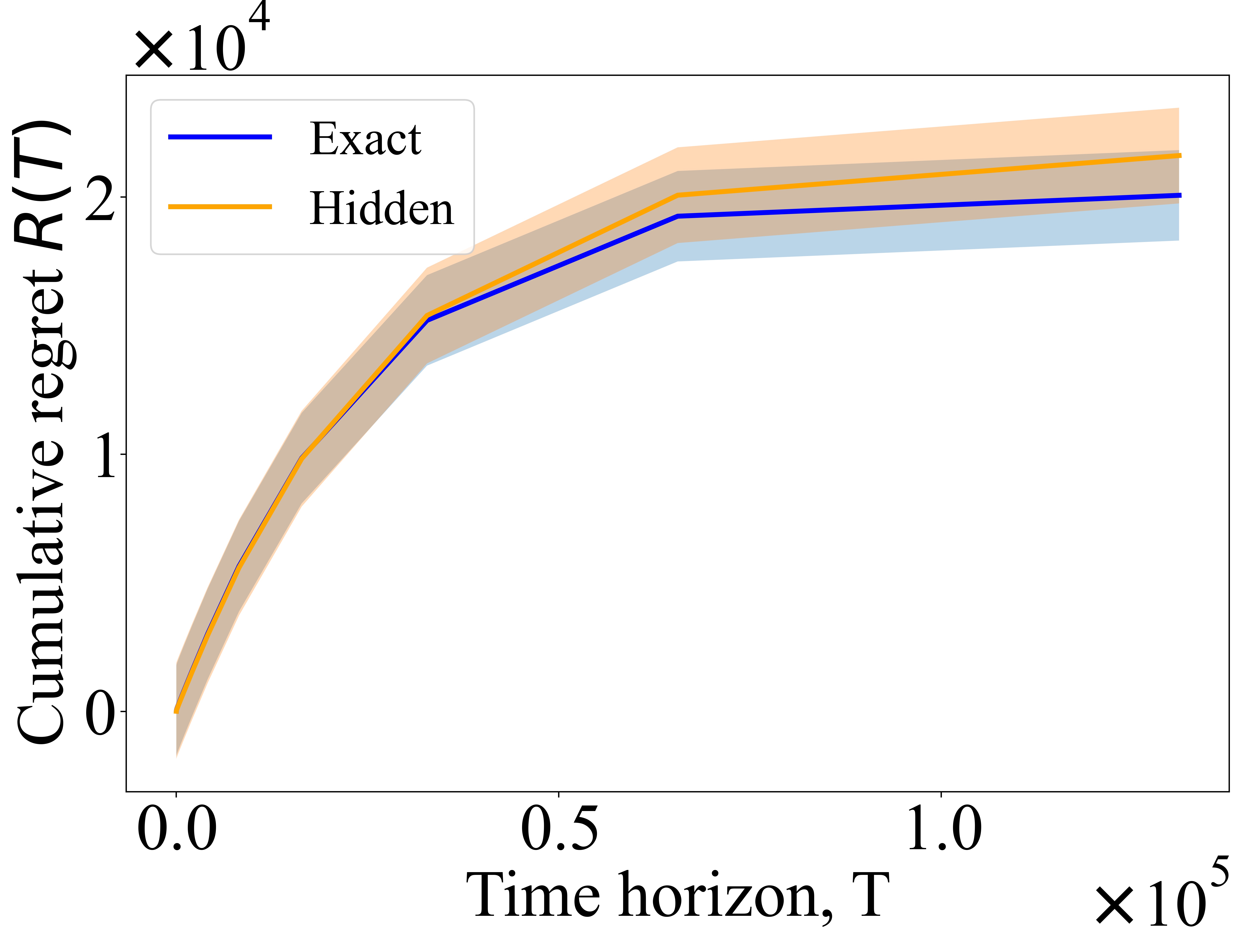}}\hspace{10 mm}
\subcaptionbox{\small \label{fig:6}}{\includegraphics[width=0.35\textwidth, height=0.25\textwidth]{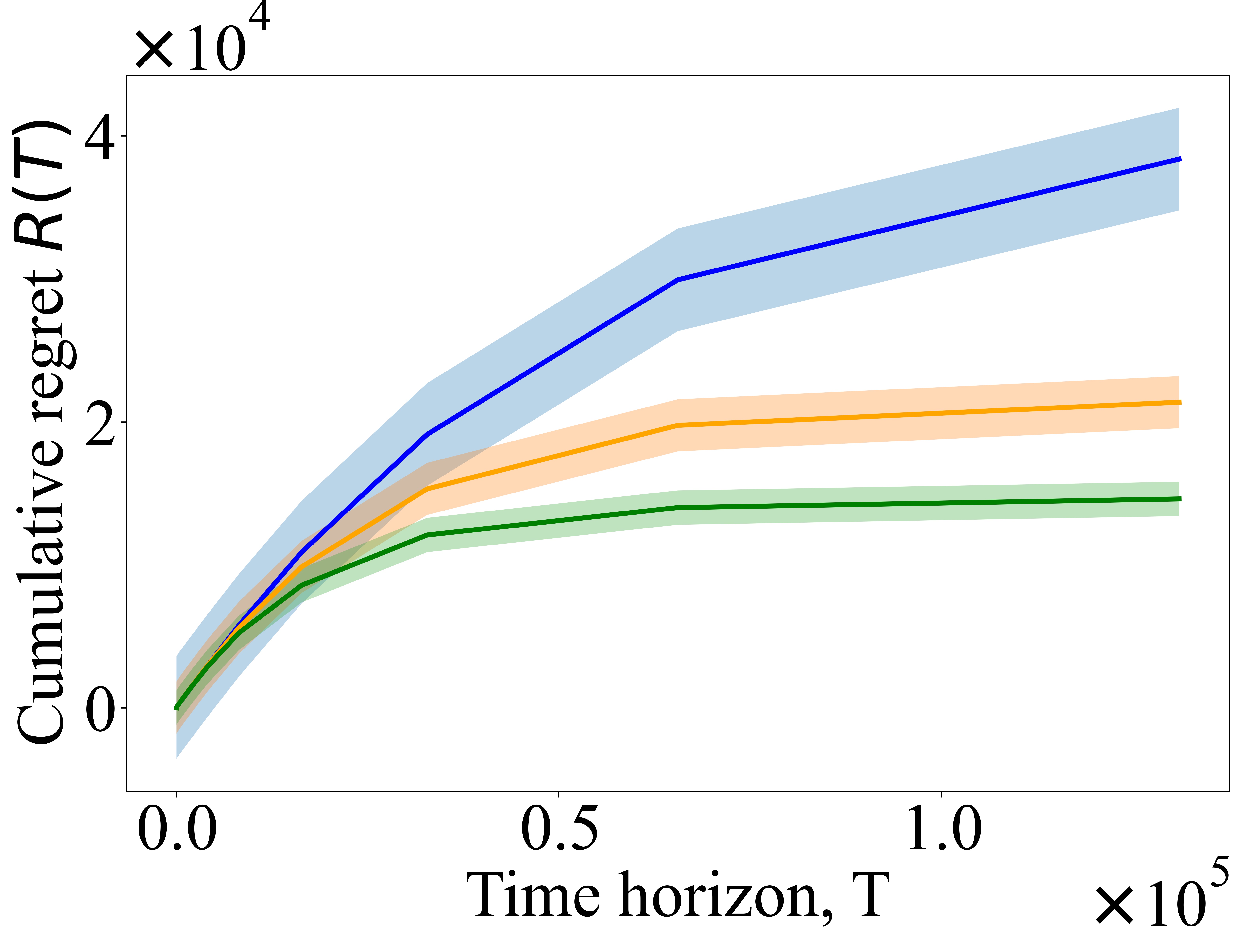}}
\vspace*{-2mm}
\caption{\small Per-agent (average) cumulative regret $R(T)$ versus time $T$. We compared the performance of our algorithm with the Fed-PE algorithm in \cite{huang2021federated} (the variant in which the actual context is observable, i.e., exact). We performed the experiments for both the synthetic data and the movielens data. {\em Synthetic data:} Figure~\ref{fig:3}  for two different variants, exact and hidden and Figure~\ref{fig:4} for different number of agents,  $M=50, 100, 150$.
{\em Movielens data:} Figure~\ref{fig:5} presents the plot for two different variants, exact and hidden, and Figure~\ref{fig:6} presents the plot for different numbers of agents,  $M=50, 100, 150$. As expected, {\em exact} outperforms the hidden setting. The figures also show that the per-agent regret decreases as the number of agents increases, validating the benefit of collaborative learning.}\label{fig:Syn1}
\vspace*{-3 mm}
\end{figure*}
We validated the performance of our algorithm using both synthetic and real-world datasets. We considered two different settings, (i) {\em exact} in which the actual feature vector $c_i$ is known to all the agents and (ii) {\em hidden} in which the feature vector is unknown and the agents only observe a distribution $\mu_i$. We compared the performance of our algorithm by varying the number of agents to see the effect of agents on the learning process. All the experiments were conducted using Python. We set $\delta = 0.1$, $T = 2^{17}$, $\fp = 2^p$, $p \in \{1, 2, \cdots, 17 \}$. Thus, the length of phase $p$ is $2^p+K$. To compare the performance for different number of agents, we set $M = \{50, 100, 150 \}$.

\textbf{Synthetic data:} In this dataset we set $K = 10$ and $d = 3$ and the feature vectors $\phi_{a, c}$  were generated randomly. From $\phi_{a, c}$, we constructed  $\psi_{a, \mu}$ by adding an observation noise. The suboptimal reward gap and $\| \phi_{a, c} \|_2$  of the data lie in $[0.2, 0.4]$ and $[0.5, 1]$, respectively. We set $\theta_a = [1, 0, 0]$.
 We present the plots showing variations of the cumulative regret with respect to time for different settings and for different numbers of agents in Figure~\ref{fig:Syn1}. Each point in the plot was averaged over 100 independent trials.  Our experimental results given in Figure~\ref{fig:3} show that the exact setting outperforms the hidden setting, as expected, since the agents observe the actual contexts. We varied the number of agents as $M = 50,100,150$ and compared the per-agent cumulative regret as shown in Figure~\ref{fig:4}. It demonstrates that as the number of agents increases, the per-agent (average) cumulative regret decreases. This is expected since when more agents work collaboratively, each agent receives more information at the end of each phase, thus accelerating the learning process. We ran for $2^{17}$ time-period. 

\textbf{Movielens data:} We used Movielens-100K data \cite{harper2015movielens} to evaluate the performance of our algorithm. We first get the rating matrix $r \in \mathbb{R}^{943 \times 1682}$ by using a non-negative matrix factorization $r = WH$, where $W \in \mathbb{R}^{943 \times 3}$, $H \in \mathbb{R}^{3 \times 1682}$. We used the k-mean clustering algorithm with $k=30$ on $H$ to cluster the action set into $30$ clusters. Thus, the number of actions $K=30$ and $\theta_a$, for $a \in \A$, is the center of the respective cluster. We set $M = 100$ by randomly selecting 100 users from the dataset. For this experiment, we noticed that the suboptimal reward gap for the data lies in $[0.01, 0.8]$ and $\| \phi_{a, c} \|_2^2$ lies in $[0.4, 0.8]$. We present the plots showing the variation of the cumulative regret with respect to time for the movielens data for different settings and for different numbers of agents in Figure~\ref{fig:Syn1}. The reward $r(a, c_i)$ is 
 bounded above by $1$, and the observation noise $\eta_i$ is set as Gaussian with zero mean and standard deviation $10^{-3}$. In this experiment, as expected, the variant that does not observe the context (hidden) is outperformed 
by the variant that uses the context observation (exact) for the estimation, as shown in Figure~\ref{fig:5}. We varied the number of agents as $M = 50,100,150$ and ran for a time period of $2^{17}$ and the plots are shown in Figure~\ref{fig:6}. The plots show that the per-agent regret decreases as the number of agents increases, validating the benefit of collaborative learning.

\vspace{-2 mm}
\section{Conclusion}\label{sec:con}
In this work, we studied a distributed and federated contextual MAB problem with unknown contexts where $M$ agents face different bandit problems and the agents' goal is to minimize the total cumulative regret. We considered a setting where the exact contexts are hidden and unobservable to the agents. In our model, each agent shares the local estimates with the central server and receives the aggregated global estimates from the server, and based on this information, each agent updates their local model.  We proposed an elimination-based algorithm, Fed-PECD, and proved the regret and communication bounds for the linearly parametrized reward function. We evaluated the performance of our algorithm using synthetic and movielens data and compared it with a baseline approach.  

\newpage

\bibliographystyle{myIEEEtran}
\bibliography{Bandits}

\begin{thebibliography}{10}
\providecommand{\url}[1]{#1}
\csname url@rmstyle\endcsname
\providecommand{\newblock}{\relax}
\providecommand{\bibinfo}[2]{#2}
\providecommand\BIBentrySTDinterwordspacing{\spaceskip=0pt\relax}
\providecommand\BIBentryALTinterwordstretchfactor{4}
\providecommand\BIBentryALTinterwordspacing{\spaceskip=\fontdimen2\font plus
\BIBentryALTinterwordstretchfactor\fontdimen3\font minus
  \fontdimen4\font\relax}
\providecommand\BIBforeignlanguage[2]{{%
\expandafter\ifx\csname l@#1\endcsname\relax
\typeout{** WARNING: IEEEtran.bst: No hyphenation pattern has been}%
\typeout{** loaded for the language `#1'. Using the pattern for}%
\typeout{** the default language instead.}%
\else
\language=\csname l@#1\endcsname
\fi
#2}}

\bibitem{cheung2013autonomous}
M.~Y. Cheung, J.~Leighton, and F.~S. Hover, ``Autonomous mobile acoustic relay
  positioning as a multi-armed bandit with switching costs,'' in \emph{IEEE/RSJ
  International Conference on Intelligent Robots and Systems}, 2013, pp.
  3368--3373.

\bibitem{srivastava2014surveillance}
V.~Srivastava, P.~Reverdy, and N.~E. Leonard, ``Surveillance in an abruptly
  changing world via multiarmed bandits,'' in \emph{IEEE Conference on Decision
  and Control (CDC)}, 2014, pp. 692--697.

\bibitem{aziz2021multi}
M.~Aziz, E.~Kaufmann, and M.-K. Riviere, ``On multi-armed bandit designs for
  dose-finding clinical trials,'' \emph{The Journal of Machine Learning
  Research}, vol.~22, no.~1, pp. 686--723, 2021.

\bibitem{anandkumar2011distributed}
A.~Anandkumar, N.~Michael, A.~K. Tang, and A.~Swami, ``Distributed algorithms
  for learning and cognitive medium access with logarithmic regret,''
  \emph{IEEE Journal on Selected Areas in Communications}, vol.~29, no.~4, pp.
  731--745, 2011.

\bibitem{li2010contextual}
L.~Li, W.~Chu, J.~Langford, and R.~E. Schapire, ``A contextual-bandit approach
  to personalized news article recommendation,'' in \emph{Proceedings of the
  19th international conference on World wide web}, 2010, pp. 661--670.

\bibitem{lattimore2020bandit}
T.~Lattimore and C.~Szepesv{\'a}ri, \emph{Bandit algorithms}.\hskip 1em plus
  0.5em minus 0.4em\relax Cambridge University Press, 2020.

\bibitem{wang2019distributed}
Y.~Wang, J.~Hu, X.~Chen, and L.~Wang, ``Distributed bandit learning:
  Near-optimal regret with efficient communication,'' \emph{arXiv preprint
  arXiv:1904.06309}, 2019.

\bibitem{lin2022distributed_1}
J.~Lin and S.~Moothedath, ``Distributed stochastic bandits with hidden
  contexts,'' \emph{European Control Conference}, 2022.

\bibitem{lin2022distributed_2}
J.~Lin and S.~Moothedath, ``Distributed stochastic bandit learning with delayed
  context observatio,'' \emph{European Control Conference}, 2022.

\bibitem{landgren2021distributed}
P.~Landgren, V.~Srivastava, and N.~E. Leonard, ``Distributed cooperative
  decision making in multi-agent multi-armed bandits,'' \emph{Automatica}, vol.
  125, p. 109445, 2021.

\bibitem{martinez2019decentralized}
D.~Mart{\'\i}nez-Rubio, V.~Kanade, and P.~Rebeschini, ``Decentralized
  cooperative stochastic bandits,'' \emph{Advances in Neural Information
  Processing Systems}, vol.~32, 2019.

\bibitem{hillel2013distributed}
E.~Hillel, Z.~S. Karnin, T.~Koren, R.~Lempel, and O.~Somekh, ``Distributed
  exploration in multi-armed bandits,'' \emph{Advances in Neural Information
  Processing Systems}, vol.~26, 2013.

\bibitem{korda2016distributed}
N.~Korda, B.~Szorenyi, and S.~Li, ``Distributed clustering of linear bandits in
  peer to peer networks,'' in \emph{International Conference on Machine
  Learning}, 2016, pp. 1301--1309.

\bibitem{tao2019collaborative}
C.~Tao, Q.~Zhang, and Y.~Zhou, ``Collaborative learning with limited
  interaction: Tight bounds for distributed exploration in multi-armed
  bandits,'' in \emph{Annual Symposium on Foundations of Computer Science
  (FOCS)}, 2019, pp. 126--146.

\bibitem{zhu2021decentralized}
J.~Zhu, E.~Mulle, C.~S. Smith, and J.~Liu, ``Decentralized multi-armed bandit
  can outperform classic upper confidence bound,'' \emph{arXiv preprint
  arXiv:2111.10933}, 2021.

\bibitem{sankararaman2019social}
A.~Sankararaman, A.~Ganesh, and S.~Shakkottai, ``Social learning in multi agent
  multi armed bandits,'' \emph{Proceedings of the ACM on Measurement and
  Analysis of Computing Systems}, vol.~3, no.~3, pp. 1--35, 2019.

\bibitem{wang2022multi}
X.~Wang, H.~Xie, and J.~Lui, ``Multi-player multi-armed bandits with finite
  shareable resources arms: Learning algorithms \& applications,'' \emph{arXiv
  preprint arXiv:2204.13502}, 2022.

\bibitem{zhu2021federated}
Z.~Zhu, J.~Zhu, J.~Liu, and Y.~Liu, ``Federated bandit: A gossiping approach,''
  in \emph{Abstract Proceedings of the 2021 ACM SIGMETRICS/International
  Conference on Measurement and Modeling of Computer Systems}, 2021, pp. 3--4.

\bibitem{shi2021federated1}
C.~Shi and C.~Shen, ``Federated multi-armed bandits,'' in \emph{Proceedings of
  the AAAI Conference on Artificial Intelligence}, vol.~35, no.~11, 2021, pp.
  9603--9611.

\bibitem{shi2021federated2}
C.~Shi, C.~Shen, and J.~Yang, ``Federated multi-armed bandits with
  personalization,'' in \emph{International Conference on Artificial
  Intelligence and Statistics}.\hskip 1em plus 0.5em minus 0.4em\relax PMLR,
  2021, pp. 2917--2925.

\bibitem{agarwal2020federated}
A.~Agarwal, J.~Langford, and C.-Y. Wei, ``Federated residual learning,''
  \emph{arXiv preprint arXiv:2003.12880}, 2020.

\bibitem{kirschner2019stochastic}
J.~Kirschner and A.~Krause, ``Stochastic bandits with context distributions,''
  \emph{Advances in Neural Information Processing Systems}, vol.~32, pp.
  14\,113--14\,122, 2019.

\bibitem{huang2021federated}
R.~Huang, W.~Wu, J.~Yang, and C.~Shen, ``Federated linear contextual bandits,''
  \emph{Advances in Neural Information Processing Systems}, vol.~34, 2021.

\bibitem{Jiabin_Shana_ACC}
J.~Lin and S.~Moothedath, ``Feature selection in distributed stochastic linear
  bandits,'' \emph{American Control Conference (ACC)}, 2022.

\bibitem{li2020federated}
T.~Li, L.~Song, and C.~Fragouli, ``Federated recommendation system via
  differential privacy,'' in \emph{2020 IEEE International Symposium on
  Information Theory (ISIT)}.\hskip 1em plus 0.5em minus 0.4em\relax IEEE,
  2020, pp. 2592--2597.

\bibitem{dubey2020differentially}
A.~Dubey and A.~Pentland, ``Differentially-private federated linear bandits,''
  \emph{Advances in Neural Information Processing Systems}, vol.~33, pp.
  6003--6014, 2020.

\bibitem{harper2015movielens}
F.~M. Harper and J.~A. Konstan, ``The movielens datasets: History and
  context,'' \emph{ACM Transactions on Interactive Intelligent Systems (TIIS)},
  vol.~5, no.~4, pp. 1--19, 2015.

\end{thebibliography}

\appendix
\renewcommand\thetheorem{\Alph{section}.\arabic{theorem}}

\subsection{Preliminaries}
To prove Theorem~\ref{Thm}, we first present the following supporting lemmas. 
\begin{lemma} \label{LT0}
Let $\zeta \in \mathbb{R}^n$ be an 1-sub-Gaussian random vector conditioned on $\F_{p-1}$ and $A \in \mathbb{R}^{n \times n}$ be an $\F_{p-1}$-measurable matrix. Let $\lambda>0$ and $det \left( I_n-2 \lambda A A^{\top} \right) > 0$. Then, we have
$$
\mathbb{E}\left[e^{\lambda\|A \zeta\|^2} \mid \F_{p-1}\right] \leq \sqrt{\frac{1}{det \left( I_n-2 \lambda A A^{\top}\right)}}.
$$
\end{lemma}
\begin{proof}
For proof of this lemma, see Lemma~4 in \cite{huang2021federated} (supplementary material).
\end{proof}
Next, we show that $\hat{r}_{a, i}^{p} - r_{a, i}$ is a conditionally sub-Gaussian random variable for all phase $p \in [H]$, any agent $i \in [M]$, and arm $a \in \A_{i}^{p-1}$. 

\begin{lemma} \label{LT1}
At phase $p \in [H]$, for any agent $i \in [M]$, and arm $a \in \A_{i}^{p-1}, \hat{r}_{a, i}^{p} - r_{a, i}$ is a conditionally sub-Gaussian random variable, i.e., $\mathbb{E}\left[\exp \left( \lambda \left( \hat{r}_{a, i}^{p} - r_{a, i}\right) \right) \mid \F_{p-1}\right] \leqslant \exp \left( \frac{\lambda^{2} \left( \sigma_{a, i}^{p} \right)^{2}}{2}\right)$, for any $\lambda \in \mathbb{R}$, where $\sigma_{a, i}^{p}:= \frac{\sqrt{10}}{\ell} \left\|\psi_{a, \mu_i}\right\|_{V_{a}^{p}}$.
\end{lemma}
\begin{proof}
Let $\xi_{a, i}^{p-1}$ be the sum of the independent sub-Gaussian noise incurred during the collaborative exploration step in phase $p-1$, i.e., $\xi_{a, i}^{p-1} := \sum_{t \in \T_{a, i}^{p-1}} \eta_{i, t}$. We know $f_{a, i}^{p-1}$ and $\xi_{a, i}^{p-1}$ are conditionally $\sqrt{f_{a, i}^{p-1}}$-sub-Gaussian random variables.
Recall the definition of local estimators, 
\begin{align}
\hat{\theta}_{a, i}^{p-1} &= \left(\frac{1}{f_{a, i}^{p-1}} \sum_{t \in \T_{a, i}^{p-1}} y_{i, t}\right) \frac{\psi_{a, \mu_i}}{\left\|\psi_{a, \mu_i}\right\|^{2}} \nonumber \\
&= \left(\phi_{a, c_i}^{\top} \theta_{a} + \frac{\xi_{a, i}^{p-1}}{f_{a, i}^{p-1}}\right) \frac{\psi_{a, \mu_i}}{\left\|\psi_{a, \mu_i}\right\|^{2}} \nonumber \\
&= \left(\psi_{a, \mu_i}^{\top} \theta_{a} + \frac{\xi_{a, i}^{p-1}}{f_{a, i}^{p-1}} + \phi_{a, c_i}^{\top} \theta_{a} - \psi_{a, \mu_i}^{\top} \theta_{a}\right) \frac{\psi_{a, \mu_i}}{\left\|\psi_{a, \mu_i}\right\|^{2}} \nonumber \\
&= \frac{\psi_{a, \mu_i} \psi_{a, \mu_i}^{\top}}{\left\|\psi_{a, \mu_i}\right\|^{2}} \theta_{a} + \frac{\psi_{a, \mu_i} \xi_{a, i}^{p-1}}{f_{a, i}^{p-1}\left\|\psi_{a, \mu_i}\right\|^{2}} + \frac{\psi_{a, \mu_i} (\phi_{a, c_i}^{\top} - \psi_{a, \mu_i}^{\top})}{\left\|\psi_{a, \mu_i}\right\|^{2}} \theta_{a}. \label{eq:label1_1}
\end{align}
We know $\hat{r}_{a, i}^p = \psi_{a, \mu_i}^{\top} \hat{\theta}_{a}^{p}$ and $V_{a}^{p} = \left( \sum_{j \in \R_{a}^{p-1}} f_{a, j}^{p-1} \frac{\psi_{a, \mu_j} \psi_{a, \mu_j}^{\top}}{\left\|\psi_{a, \mu_j}\right\|^{2}}\right)^{\dagger}$. Thus
\begin{align}
&\hat{r}_{a, i}^{p} - r_{a, i} = \psi_{a, \mu_i}^{\top} \hat{\theta}_{a}^{p} - \phi_{a, c_i}^{\top} \theta_{a} \nonumber \\
&= \psi_{a, \mu_i}^{\top} V_{a}^{p} (\sum_{j \in \R_{a}^{p-1}} f_{a, j}^{p-1} \hat{\theta}_{a, j}^{p-1} ) - \phi_{a, c_i}^{\top} \theta_{a} \nonumber 
\end{align}
\begin{align}
&\scalemath{0.9}{= \psi_{a, \mu_i}^{\top} V_{a}^{p} \bigg(\sum_{j \in \R_{a}^{p-1}} f_{a, j}^{p-1} \big(\frac{\psi_{a, \mu_j} \psi_{a, \mu_j}^{\top}}{\left\|\psi_{a, \mu_j}\right\|^{2}} \theta_{a}+ \frac{\psi_{a, \mu_j} \xi_{a, j}^{p-1}}{f_{a, j}^{p-1}\left\|\psi_{a, \mu_j}\right\|^{2}}} \nonumber\\
&\scalemath{0.9}{+ \frac{\psi_{a, \mu_j} (\phi_{a, c_j}^{\top} - \psi_{a, \mu_j}^{\top})}{\left\|\psi_{a, \mu_j}\right\|^{2}} \theta_{a}\big)\bigg) - \phi_{a, c_i}^{\top} \theta_{a}} \label{eq:label1_2} \\
&\scalemath{0.9}{\leqslant \psi_{a, \mu_i}^{\top} V_{a}^{p}\left(\sum_{j \in \R_{a}^{p-1}} \frac{e_{a, j}}{\left\|\psi_{a, \mu_j}\right\|} (\xi_{a, j}^{p-1} + 2f_{a, j}^{p-1})\right) + (\psi_{a, \mu_i}^{\top} - \phi_{a, c_i}^{\top}) \theta_{a}}\label{eq:label1_add0}\\
&\scalemath{0.9}{\leqslant \psi_{a, \mu_i}^{\top} V_{a}^{p} \left(\sum_{j \in \R_{a}^{p-1}} \frac{e_{a, j}}{\left\|\psi_{a, \mu_j}\right\|} (\xi_{a, j}^{p-1} + 2f_{a, j}^{p-1})\right) + 2} \label{eq:label1_4}
\end{align}
Eq.~\eqref{eq:label1_2} follows from Eq.~\eqref{eq:label1_1}. 
Eq.~\eqref{eq:label1_add0} follows from $\left( \phi_{a, c_j}^{\top} - \psi_{a, \mu_j}^{\top}\right) \theta_{a} \leqslant 2$,  $e_{a, j} = \frac{\psi_{a, j}}{\left\| \psi_{a, j}\right\|}$, and $V_{a}^{p} = \left( \sum_{j \in \R_{a}^{p-1}} f_{a, j}^{p-1} \frac{\psi_{a, \mu_j} \psi_{a, \mu_j}^{\top}}{\left\|\psi_{a, \mu_j}\right\|^{2}}\right)^{\dagger}$. Finally, 
Eq.~\eqref{eq:label1_4} follows from $\left( \phi_{a, c_i}^{\top} - \psi_{a, \mu_i}^{\top}\right) \theta_{a} \leqslant 2$. 

Given $f_{a, i}^{p-1}, \xi_{a, i}^{p-1}$ are conditionally $\sqrt{f_{a, i}^{p-1}}$-sub-Gaussian random variables. Thus $\left( \xi_{a, j}^{p-1} + 2f_{a, j}^{p-1}\right)$ is a conditionally $\sqrt{\left( 1^{2} + 2^{2}\right) f_{a, j}^{p-1}} = \sqrt{5f_{a, j}^{p-1}}$ sub-Gaussian random variable. Also Eq.~\eqref{eq:label1_add0} is a linear combination of $\left( \xi_{a, j}^{p-1} + 2f_{a, j}^{p-1}\right)$ and $\left( \psi_{a, \mu_i}^{\top} - \phi_{a, c_i}^{\top}\right) \theta_{a}$. Thus, given $\F_{p-1}$, $\hat{r}_{a, i}^{p}-r_{a, i}$ is a conditionally sub-Gaussian random variable, whose parameter can be bounded as
\begin{align}
&\scalemath{0.9}{\mathbb{E}_{c_{i} \sim \mu_{i}} [\hat{r}_{a, i}^{p} - r_{a, i}]^2}\nonumber\\
&\scalemath{0.9}{ \leqslant \mathbb{E}_{c_{i} \sim \mu_{i}} \hspace*{-1 mm}\left[\psi_{a, \mu_i}^{\top} V_{a}^{p}\left(\sum_{j \in \R_{a}^{p-1}}\hspace*{-1 mm} \frac{e_{a, j}}{\left\|\psi_{a, \mu_j}\right\|} (\xi_{a, j}^{p-1} \hspace*{-1 mm}+\hspace*{-0.5 mm} 2f_{a, j}^{p-1})\hspace*{-1 mm}\right)\hspace*{-0.5 mm} +\hspace*{-0.5 mm} (\psi_{a, \mu_i}^{\top}\hspace*{-0.5 mm} -\hspace*{-0.5 mm} \phi_{a, c_i}^{\top}) \theta_{a}\right]^2} \nonumber \\ 
&\scalemath{0.9}{\leqslant 2 \sum_{j \in \R_{a}^{p-1}} \left[\psi_{a, \mu_i}^{\top} V_{a}^{p} \frac{e_{a, j}}{\left\|\psi_{a, \mu_j}\right\|} ( \xi_{a, j}^{p-1} + 2f_{a, j}^{p-1} )\right]^{2}} \label{eq:label1_add2} \\ 
&\scalemath{0.9}{= 2 \sum_{j \in \R_{a}^{p-1}} \left(\psi_{a, \mu_i}^{\top} V_{a}^{p} \frac{e_{a, j}}{\left\|\psi_{a, \mu_j}\right\|}\right)^{2} \cdot 5 f_{a, j}^{p-1}} \nonumber \\ 
&\scalemath{0.9}{= \sum_{j \in \R_{a}^{p-1}} \psi_{a, \mu_i}^{\top} V_{a}^{p} f_{a, j}^{p-1} e_{a, j} e_{a, j}^{\top} V_{a}^{p} \psi_{a, \mu_i} \frac{10}{\left\|\psi_{a, \mu_j}\right\|^{2}}} \nonumber \\ 
&\scalemath{0.9}{ \leqslant \frac{10}{\ell^{2}} \psi_{a, \mu_i}^{\top} V_{a}^{p} \psi_{a, \mu_i}= (\sigma_{a, i}^{p})^{2}.} \label{eq:label1_5}
\end{align}

Eq.~\eqref{eq:label1_add2} follows from $\left( a+b\right)^2 \leqslant 2a^2 + 2b^2$ and uses the fact that $ \psi_{a, \mu_{i}}=\mathbb{E}_{c_{i} \sim \mu_{i}}\left[\phi_{a, c_{i}} \mid \F_{p-1}, \mu_{i}, c_i\right]$.
Eq.~\eqref{eq:label1_5} follows from $\left\|\psi_{a, \mu_j}\right\| \geqslant \ell$, $A\left( A\right)^{\dagger}A = A$ for all $i$, $a$.
\end{proof}

\begin{lemma} \label{LTadd1}
From the definition of $\sigma_{a, i}^{p}:= \frac{\sqrt{10}}{\ell} \left\|\psi_{a, \mu_i}\right\|_{V_{a}^{p}}$, where $V_a^p$ is a symmetric matrix, for all $a \in [K], i\in [M],$ and $p \in [H]$, we have $\frac{\l \sigma_{a, i}^{p}}{\sqrt{10 L}} \geqslant 1$. 
\end{lemma}
\begin{proof} We have
\[
\left\|\psi_{a, \mu_i}\right\|_{V_{a}^{p}}^2 - \psi_{a, \mu_i}^{\top} \psi_{a, \mu_i} = \psi_{a, \mu_i}^{\top} \left( V_{a}^{p} - I\right) \psi_{a, \mu_i}.
\]
Since $V_a^p$ is a symmetric matrix, $V_a^p - I$ is also a symmetric matrix. By using the eigendecomposition of the symmetric matrix $V_a^p - I$, we know there exists an orthogonal matrix $U$ such that $\Lambda = diag\left( \lambda_1, \lambda_2, \cdots, \lambda_n\right) = U\left( V_a^p - I\right)U^{\top}$, where $\lambda_i > 0$ is the $i$-th eigenvalue of $V_a^p - I$. Therefore we get 
$$
\left\|\psi_{a, \mu_i}\right\|_{V_{a}^{p}}^2 - \psi_{a, \mu_i}^{\top} \psi_{a, \mu_i} = \psi_{a, \mu_i}^{\top} U^{\top} \Lambda U \psi_{a, \mu_i} = \left\|\Lambda^{\frac{1}{2}} U \psi_{a, \mu_i}\right\|_2^2 \geqslant 0
$$
and 
\begin{align*}
\frac{1}{L} \left( \left\|\psi_{a, \mu_i}\right\|_{V_{a}^{p}}^2 - \psi_{a, \mu_i}^{\top} \psi_{a, \mu_i}\right) &= \frac{\l^2 \left( \sigma_{a, i}^{p}\right)^2}{10 L} - \frac{\psi_{a, \mu_i}^{\top} \psi_{a, \mu_i}}{L}\\
& \geqslant \frac{\l^2 \left( \sigma_{a, i}^{p}\right)^2}{10 L} - 1 \geqslant 0.
\end{align*}
Thus $\frac{\l \sigma_{a, i}^{p}}{\sqrt{10 L}} \geqslant 1$ and this completes the proof.
\end{proof}

Recall the event $\E\left( \alpha\right)$
$$
\E\left( \alpha\right) := \{\exists p \in [H], i \in [M], a \in \A_{i}^{p-1}, \left|\hat{r}_{a, i}^{p}-r_{a, i}\right| \geqslant u_{a, i}^{p} = \alpha \sigma_{a, i}^{p}\},
$$
where $\scalemath{0.9}{\alpha = \min \left\{\sqrt{2 \log \left( 2 M K H / \delta\right)}, \sqrt{2 \log \left( K H / \delta\right) + d \log \left( k e\right)}\right\} }$, and $\sigma_{a, i}^{p}:= \frac{\sqrt{10}}{\ell} \left\|\psi_{a, \mu_i}\right\|_{V_{a}^{p}}$.
We refer to $\E\left( \alpha\right)$ as a ``bad" event and $\E^{c}\left( \alpha\right)$ as a "good" event. We define $\F_{p} := \left\{\hat{\theta}_{a, i}^{p}\right\}_{p \in [H], i \in [M], a \in \A_{i}^{p}}$, the information available at the end of phase $p$. 

In Lemma~\ref{LT2}, we show that the probability of bad events is only less than $\delta$. 
%
%
\subsection{Proof of Lemma~\ref{LT2}}\label{sec:LT2_proof}
Let $ \alpha = \min \left\{\alpha_{1}, \alpha_{2}\right\} $, where 
\begin{align}
&\alpha_{1} = \sqrt{2 \log \left( 2 M K H / \delta\right)} \nonumber \\ 
&\alpha_{2} = \sqrt{2 \log \left( K H / \delta\right) + d \log \left( k e\right)} \nonumber 
\end{align}
From Theorem~\ref{Thm}, we know $k > 1$ and $k \geqslant \left\{\frac{\alpha_{2}^{2}}{d}\right\}$. This choice of $k$ requires $\alpha_2^2 \geqslant d$. Based on the definition of $\E \left( \alpha\right)$, we have
\[
\E \left( \alpha\right) = \E \left( \min \left\{\alpha_{1}, \alpha_{2}\right\} \right) \supset \E \left( \max \left\{\alpha_{1}, \alpha_{2}\right\} \right),
\]
which implies that $\mP [\E \left( \alpha\right)] = \max \left( \mP [\E \left( \alpha_{1} \right) ], \mP [\E \left( \alpha_{2}\right)]\right)$. Therefore, we need to prove that $\mP [\E \left( \alpha_{i} \right)] \leqslant \delta$, for all $i \in \left\{1, 2\right\}$. In the following, we bound $\mP [\E \left( \alpha_{1} \right)]$ and $\mP [\E \left( \alpha_{2} \right)]$ separately. 

(i) Bound $\mP [\E\left( \alpha_{1} \right)]$. 
Based on Lemma~\ref{LT1} and Hoeffding's inequality, we have
\begin{align}
\mP [|\hat{r}_{a, i}^{p} - r_{a, i}| \geqslant \alpha_{1} \sigma_{a, i}^{p} \mid \F_{p} ] &\leqslant 2 \exp \left( - \frac{\alpha_{1}^{2} \left( \sigma_{a, i}^{p}\right)^{2}}{2 \left( \sigma_{a, i}^{p}\right)^2}\right) \nonumber \\ 
&= 2 \exp \left( - \frac{\alpha_{1}^{2}}{2}\right)= \frac{\delta}{M K H}. \nonumber
\end{align}

Then, by using the union bound, we get
\[
\begin{aligned}
\mP [\E \left( \alpha_{1}\right) ] &= \mP [\exists p \in [H], i \in [M], a \in \A_{i}^{p-1}, |\hat{r}_{a, i}^{p} - r_{a, i}| \geqslant \alpha_{1} \sigma_{a, i}^{p} ] \\
& \leqslant \sum_{p \in [H]} \sum_{i \in [M]} \sum_{a \in \A_{i}^{p-1}} \mP [|\hat{r}_{a, i}^{p} - r_{a, i}| \geqslant \alpha_{1} \sigma_{a, i}^{p} \mid \F_{p-1} ] \\
& \leqslant H M K \frac{\delta}{M K H} = \delta.
\end{aligned}
\]

(ii) Bound $\mP [\E \left( \alpha_{2} \right)]$. 
\begin{align}
&|\hat{r}_{a, i}^{p} - r_{a, i}| \leqslant \left|\psi_{a, \mu_i}^{\top} V_{a}^{p} \left(\sum_{j \in \R_{a}^{p-1}} \frac{e_{a, j}}{\left\|\psi_{a, \mu_j}\right\|} \left(\xi_{a, j}^{p-1} + 2f_{a, j}^{p-1}\right) \right)  \right|+ 2 \nonumber \\ 
& \leqslant \left\|\psi_{a, \mu_i}\right\|_{V_{a}^{p}} \cdot \left\|\sum_{j \in \R_{a}^{p-1}} \frac{e_{a, j}}{\left\|\psi_{a, \mu_j}\right\|} \left( \xi_{a, j}^{p-1} + 2f_{a, j}^{p-1}\right)\right\|_{V_{a}^{p}} + 2 \label{eq:label2_1} \\ 
&\leqslant \frac{\sigma_{a, i}^{p}}{\sqrt{10}} \left\|\sum_{j \in \R_{a}^{p-1}} \frac{e_{a, j}}{\left\|\psi_{a, \mu_j}\right\|} \l \left( \xi_{a, j}^{p-1} + 2 f_{a, j}^{p-1}\right)\right\|_{V_{a}^{p}} + \frac{2 \l \sigma_{a, i}^{p}}{\sqrt{10 L}} \label{eq:label2_add3}
\end{align} 

Eq.~\eqref{eq:label2_1} follows from Cauchy-Schwarz inequality. Eq.~\eqref{eq:label2_add3} follows from Lemma~\ref{LTadd1}. Therefore, when $\E \left( \alpha_{2}\right)$ happens, by the definition of the event $\E \left( \alpha\right)$, we have 
$$
\frac{1}{\sqrt{10}} \left\|\sum_{j \in \R_{a}^{p-1}} \frac{e_{a, j}}{\left\|\psi_{a, \mu_j}\right\|} \l \left( \xi_{a, j}^{p-1} + 2f_{a, j}^{p-1}\right)\right\|_{V_{a}^{p}} + \frac{2 \l}{\sqrt{10 L}} \geqslant \alpha_{2}
$$
hold for some phase $p$ and arm $a$.

Let us define the term 
$$
\Xi_{a, p} := \frac{1}{10} \left\|\sum_{j \in \R_{a}^{p-1}} \frac{e_{a, j}}{\left\|\psi_{a, \mu_j}\right\|} \l \left( \xi_{a, j}^{p-1} + 2f_{a, j}^{p-1}\right)\right\|_{V_{a}^{p}}^{2}
$$
for a given phase $p$ and arm $a$. Note that
\begin{align*}
\scalemath{0.9}{\Xi_{a, p} = \frac{1}{2} \sum_{i, j \in \R_{a}^{p-1}}\Biggl( \frac{\l \left( \xi_{a, i}^{p-1} + 2f_{a, i}^{p-1}\right) e_{a, i}^{\top}}{\sqrt{5} \left\|\psi_{a, \mu_i}\right\|} \cdot \left( \sum_{k \in \R_{a}^{p-1}} f_{a, k}^{p-1} e_{a, k} e_{a, k}^{\top}\right)^{\dagger}}\\
\scalemath{0.9}{\cdot \frac{\l \left( \xi_{a, j}^{p-1} + 2f_{a, j}^{p-1}\right) e_{a, j}}{\sqrt{5} \left\|\psi_{a, \mu_j}\right\|}\Biggr).}
\end{align*}
In the following proof, we use matrix form for $\Xi_{a, p}$. Since $f_{a, i}^{p-1}$ may equal 0 under the Block-Coordinate Ascent (BCA) algorithm \cite{huang2021federated} (Algorithm~3 in supplementary), and when this occurs, agent $i$ does not choose arm $a$ during the collaborative exploration step in phase $p$, even though $a$ is in the active arm set. In this case, $\xi_{a, i}^{p-1} = 0$. Therefore, we define $\zeta_{a, i}$ as follows.
$$
\zeta_{a, i} = \left\{\begin{array}{cl}
\frac{\l \left( \xi_{a, i}^{p-1} + 2f_{a, i}^{p-1}\right)}{\sqrt{5} \sqrt{f_{a, i}^{p-1}}\left\|\psi_{a, \mu_i}\right\|}, & \text{if } f_{a, i}^{p-1} \neq 0, \\
0, & \text{if } f_{a, i}^{p-1} = 0 .
\end{array}\right.
$$
Then, we know $ \{\zeta_{a, i}\}_{i \in \R_a^{p-1}}$ are conditionally independent 1-sub-Gaussian random variables.
we define vector $\zeta := \{\zeta_{a, i}\}_{i \in \R_{a}^{p-1}}$ and matrix $A := \left( a_{i, j}\right)_{i, j \in \R_{a}^{p-1}}$ where
$$
a_{i, j} = \sqrt{f_{a, i}^{p-1}} e_{a, i}^{\top} \left( \sum_{k \in \R_{a}^{p-1}} f_{a, k}^{p-1} e_{a, k} e_{a, k}^{\top}\right)^{\dagger} e_{a, j} \sqrt{f_{a, j}^{p-1}}. 
$$
Here $A$ is a symmetric matrix and
\begin{align*}
&\sum_{k \in \R_{a}^{p-1}} a_{i, k} a_{k, j} = \sum_{k \in \R_{a}^{p-1}} \Biggl(\sqrt{f_{a, i}^{p-1}} e_{a, i}^{\top} \left( \sum_{k \in \R_{a}^{p-1}} f_{a, k}^{p-1} e_{a, k} e_{a, k}^{\top}\right)^{\dagger}\\
&\cdot e_{a, k} \sqrt{f_{a, k}^{p-1}} \sqrt{f_{a, k}^{p-1}} e_{a, k}^{\top} \left( \sum_{k \in \R_{a}^{p-1}} f_{a, k}^{p-1} e_{a, k} e_{a, k}^{\top}\right)^{\dagger} e_{a, j} \sqrt{f_{a, j}^{p-1}}\Biggr)
\end{align*}
\begin{align*}
&= \sqrt{f_{a, i}^{p-1}} e_{a, i}^{\top} \left( \sum_{k \in \R_{a}^{p-1}} f_{a, k}^{p-1} e_{a, k} e_{a, k}^{\top}\right)^{\dagger} \cdot\left( \sum_{k \in \R_{a}^{p-1}} f_{a, k}^{p-1} e_{a, k} e_{a, k}^{\top}\right)\\
&\hspace*{3 cm}\cdot \left( \sum_{k \in \R_{a}^{p-1}} f_{a, k}^{p-1} e_{a, k} e_{a, k}^{\top}\right)^{\dagger} e_{a, j} \sqrt{f_{a, j}^{p-1}} \\
&= \sqrt{f_{a, i}^{p-1}} e_{a, i}^{\top} V_{a}^{p}\left( V_{a}^{p}\right)^{\dagger} V_{a}^{p} e_{a, j} \sqrt{f_{a, j}^{p-1}} \\
&= \sqrt{f_{a, i}^{p-1}} e_{a, i}^{\top} V_{a}^{p} e_{a, j} \sqrt{f_{a, j}^{p-1}}= a_{i, j}. 
\end{align*}
Therefore, we know $A^{2} = A$, which implies that all the eigenvalues of $A$ are either $1$ or $0$.
Also, we can get
$$
\begin{aligned}
&\mathrm{trace}\left( A\right) = \sum_{i \in \R_{a}^{p-1}} a_{i, i} \\
&= \sum_{i \in \R_{a}^{p-1}} \sqrt{f_{a, i}^{p-1}} e_{a, i}^{\top} \left( \sum_{k \in \R_{a}^{p-1}} f_{a, k}^{p-1} e_{a, k} e_{a, k}^{\top}\right)^{\dagger} e_{a, i} \sqrt{f_{a, i}^{p-1}} \\
&= V_{a}^{p} \left( \sum_{i \in \R_{a}^{p-1}} f_{a, i}^{p-1} e_{a, i} e_{a, i}^{\top}\right)^{\dagger} = V_{a}^{p} \left( V_{a}^{p}\right)^{\dagger}= \mathrm{rank}\left( V_{a}^{p}\right)= d_{a}^{p}.
\end{aligned}
$$
We know that all eigenvalues are either $1$ or $0$. Since the $\mathrm{rank}\left( A\right) = d_{a}^{p}$, there are exactly $d_{a}^{p}$ eigenvalues that are equal to $1$, and the rest of the eigenvalues are all $0$. Thus, $\mathrm{rank}\left( A\right) = d_{a}^{p} \leqslant d$.

From the definition of $\zeta$ and since $A^{2} = A$, we have $\Xi_{a, p} = \frac{1}{2} \zeta^{\top} A \zeta = \frac{1}{2} \|A \zeta\|^{2}$, where $\zeta$ is a conditionally 1 -sub-Gaussian random vector. Here, $A$ is an $\F_{p}$-measurable matrix. Therefore, for any $\lambda \in\left( 0,1 / 2\right)$, we have
\begin{align}
&\scalemath{0.9}{\mP [\left( \sqrt{\Xi_{a, p}} + \frac{2 \l}{\sqrt{10 L}}\right)^2 \geqslant \alpha_{2}^{2} \mid \F_{p-1}] \leqslant \mP [2 \Psi_{a, p} + \frac{4 \l^2}{5 L} \geqslant \alpha_{2}^{2} \mid \F_{p-1}]} \label{eq:label2_add4} \\
&= \mP [\|A \zeta\|^{2} + \frac{4 \l^2}{5 L} \geqslant \alpha_{2}^{2} \mid \F_{p-1}] \nonumber \\
&= \mP [e^{\lambda\|A \zeta\|^2 + \frac{4 \l^2 \lambda}{5 L}} \geqslant e^{\lambda \alpha_{2}^{2}} \mid \F_{p-1}] \nonumber \\
& \leqslant e^{- \lambda \alpha_{2}^{2}} \mathbb{E} [e^{\lambda\|A \zeta\|^2 + \frac{4 \l^2 \lambda}{5 L}} \mid \F_{p-1}] \label{eq:label2_2} \\
& \leqslant e^{- \lambda \alpha_{2}^{2} + \frac{4 \l^2 \lambda}{5 L}} \sqrt{\frac{1}{det \left( I_{d} - 2 \lambda A^{2}\right)}}\label{eq:label2_3}\\
&= e^{-\lambda \alpha_{2}^{2} + \frac{4 \l^2 \lambda}{5 L}} \left( 1 - 2 \lambda\right)^{-d_{a}^{p} / 2}\label{eq:label2_4}\\
& \leqslant e^{- \lambda \alpha_{2}^{2} + \frac{4 \l^2 \lambda}{5 L}} \left( 1 - 2 \lambda\right)^{-d / 2} \nonumber
\end{align}

Eq.~\eqref{eq:label2_2} follows from Markov's inequality, Eq.~\eqref{eq:label2_3} follows from Lemma~\ref{LT0}, Eq.~\eqref{eq:label2_4} follows from the fact that the eigenvalues of $A$ are either 1 or 0 and there are exactly $d_{a}^{p}$ number of $1$'s.

By choosing $\lambda = \frac{\alpha_{2}^{2} - d}{2 \alpha_{2}^{2}} \in \left( 0, \frac{1}{2}\right)$, we have
$$
\scalemath{0.9}{\mP [\left( \sqrt{\Xi_{a, p}} + \frac{2 \l}{\sqrt{10 L}}\right)^2 \geqslant \alpha_{2}^{2} \mid \F_{p-1}] \leqslant  \left( \frac{\alpha_{2}^{2}}{d}\right)^{d / 2} e^{-\frac{\alpha_{2}^{2}-d}{2} + \frac{4 \l^2 \lambda}{5 L}} \leqslant \frac{\delta}{K H}}
$$
The last inequality follows by the following analysis:
$$
\begin{aligned}
& \left( \frac{\alpha_{2}^{2}}{d}\right)^{d / 2} e^{-\frac{\alpha_{2}^{2}-d}{2} + \frac{4 \l^2 \lambda}{5 L}} \leqslant \frac{\delta}{K H} \\
&\scalemath{0.9}{\Leftrightarrow \frac{d}{2} \left( 2 \log \alpha_{2} - \log d\right) + \frac{d}{2} + \frac{4 \l^2 \lambda}{5 L} - \frac{1}{2}\left( \sqrt{2 \log \left( K H / \delta\right) + d \log \left( k e\right)}\right)^{2}}\\
& \hspace*{4 cm} \leqslant \log  \left( \frac{\delta}{K H}\right) \\
&\Leftrightarrow d \log \alpha_{2} \leqslant \frac{d \log \left( d k\right)}{2} - \frac{4 \l^2 \lambda}{5 L} \\
&\Leftrightarrow \alpha_{2}^{2} \leqslant d k e^{- \frac{8 \l^2 \lambda}{5 d L}} \\
&\Leftrightarrow \alpha_{2}^{2} \leqslant d k.
\end{aligned}
$$
The last step is assured by the definition of $k$.
We finish the proof by using the union bound over phase $p$ and arm $a$ that $\mP [\E \left( \alpha_{2}\right)] \leqslant H K \frac{\delta}{H K} = \delta$.
\qed

 Lemma~\ref{LT4} presents a bound on the regret for the good events. The proofs of both these results follows a similar approach as in \cite{huang2021federated} (Lemmas~13 and~14 in supplementary material) and uses the earlier lemmas of this paper (Lemmas~\ref{LT0}-\ref{LT2}).
 
\subsection{Proof of Lemma~\ref{LT4}}\label{app-2}
The proof follows similar steps as in \cite{huang2021federated}. The per-phase regret bound is different, in terms of a factor, from that in \cite{huang2021federated} as shown in the earlier lemmas of this paper. We present the proof below  for completeness

When $\E^{c}\left( \alpha\right)$ occurs, the following is true  for all $p \in [H], i \in [M],$ and $a \in \A_{i}^{p-1}$.
\[\left|\hat{r}_{a, i} - r_{a, i}\right| \leqslant u_{a, i}^{p} = \alpha \sigma_{a, i}^{p} = \alpha \frac{\sqrt{10}}{\ell} \left\|\psi_{a, \mu_i}\right\|_{V_{a}^{p}}.\]
Recall the definition of the active action set  $\A_{i}^{p}$ (line~\ref{line:action-set} of Algorithm~\ref{alg: client}). Thus if any active  action $a \in \A_{i}^{p}$ is selected in phase $p$, then we will incur a regret that is  upper bounded by
\begin{align}
\Delta_{a, i} &= r_{a_{i}^{\star}, i} - r_{a, i} \nonumber \\
&= r_{a_{i}^{\star}, i} - \hat{r}_{a_{i}^{\star}, i} - r_{a, i} + \hat{r}_{a, i} + \hat{r}_{a_{i}^{\star}, i} - \hat{r}_{a, i} \nonumber \\
&\leqslant u_{a_{i}^{\star}, i}^{p} + u_{a, i}^{p} + \hat{r}_{\hat{a_{i}}, i} - \hat{r}_{a, i} \label{eq:label3_1} \\
&\leqslant u_{a_{i}^{\star}, i}^{p} + u_{a, i}^{p} + u_{a, i}^{p} + u_{\hat{a}_{i}, i}^{p} \label{eq:label3_2} \\
&= u_{\hat{a}_{i}, i}^{p} + u_{a_{i}^{\star}, i}^{p} + 2 u_{a, i}^{p} \nonumber \\
&\leqslant 4 \max _{a \in \A_{i}^{p}} u_{a, i}^{p}. \label{eq:label3_3} 
\end{align}

Eq.~\eqref{eq:label3_1} follows from $\hat{r}_{a_{i}^{\star}, i} \leqslant \hat{r}_{\hat{a}_{i}, i}$. 
Eq.~\eqref{eq:label3_2} follows from $\hat{r}_{a, i}^{p} + u_{a, i}^{p} \geqslant \hat{r}_{\hat{a_{i}}, i}^{p} - u_{\hat{a}_{i}, i}^{p}$. 
Eq.~\eqref{eq:label3_3} follows from Lemma~\ref{LT3} that any optimal arm will never be eliminated

We know that the length of the phase in the Fed-PECD algorithm  is fixed as $f^{p} + K$ for $p \in [H]$. Then, given that the good event $\E^{c}\left( \alpha\right)$ happens, the total regret incurred during phase $p$, denoted as $R_p$, can be bounded as follows.
\begin{align}
R_p &\leqslant \sum_{i \in[M]} 4 \max _{a \in \A_{i}^{p}} u_{a, i}^{p} \left( f^{p} + K\right) \nonumber \\
&= 4 \left( f^{p} + K\right) \sqrt{\left( \sum_{i \in[M]} 1 \cdot \max_{a \in \A_{i}^{p}} u_{a, i}^{p}\right)^{2}} \nonumber \\ 
&\leqslant 4 \left( f^{p} + K\right) \sqrt{\left( \sum_{i \in[M]} 1^{2}\right) \left( \sum_{i \in[M]} \max_{a \in \A_{i}^{p}} \left( u_{a, i}^{p}\right)^{2}\right)} \label{eq:label3_4} \\ 
&= 4 \left( f^{p} + K\right) \sqrt{M \sum_{i \in[M]} \max_{a \in \A_{i}^{p}} \left( u_{a, i}^{p}\right)^{2}} \nonumber \\ 
&\leqslant \frac{4 \sqrt{10} \alpha L}{\l} \left( f^{p} + K\right) \sqrt{M \sum_{i \in[M]} \max_{a \in \A_{i}^{p}} \frac{\psi_{a, \mu_i}^{\top}}{\left\|\psi_{a, \mu_i}\right\|_{2}} V_{a}^{p} \frac{\psi_{a, \mu_i}}{\left\|\psi_{a, \mu_i}\right\|_{2}}} \label{eq:label3_5}
\end{align}
\begin{align}
&\hspace*{-4 mm}\leqslant \frac{4 \sqrt{10} \alpha L}{\l}\hspace*{-1 mm} \left( f^{p} + K\right)\hspace*{-2 mm} \sqrt{M\hspace*{-1 mm}\sum_{i \in[M]} \hspace*{-0.5 mm}\max_{a \in \A_{i}^{p}} e_{a, i}^{\top} \left( \sum_{j \in \R_{a}^{p}}\hspace*{-1 mm}\left\lceil\pi_{a, j} f^{p-1}\right\rceil e_{a, j} e_{a, j}^{\top}\right)^{\dagger}\hspace*{-1 mm} e_{a, i}} \label{eq:label3_6} \\ 
&\hspace*{-4 mm}\leqslant \frac{4 \sqrt{10} \alpha L}{\l} \left( f^{p} + K\right)\hspace*{-1 mm} \sqrt{\frac{M}{f^{p-1}} \hspace*{-1 mm}\sum_{i \in[M]} \max_{a \in \A_{i}^{p}} e_{a, i}^{\top} \left( \sum_{j \in \R_{a}^{p}} \pi_{a, j} e_{a, j} e_{a, j}^{\top}\right)^{\dagger} e_{a, i}} \label{eq:label3_7} \\ 
&\leqslant \frac{4 \sqrt{10} \alpha L}{\l} \left( f^{p} + K\right) \sqrt{\frac{d K M}{f^{p-1}}} \label{eq:label3_8} \\ 
&= \frac{4 \sqrt{10} \alpha L}{\l} \sqrt{d K M} \frac{f^{p} + K}{\sqrt{f^{p-1}}} \nonumber 
\end{align}

Eq.~\eqref{eq:label3_4} follows from Cauchy-Schwarz inequality that $\left( \sum_{i=1}^{n} u_{i} v_{i}\right)^{2} \leqslant \left( \sum_{i=1}^{n} u_{i}^{2}\right) \left( \sum_{i=1}^{n} v_{i}^{2}\right)$. 
Eq.~\eqref{eq:label3_5} follows from $u_{a, i}^{p} := \frac{\sqrt{10} \alpha}{\l} \left\|\psi_{a, \mu_i}\right\|_{V_{a}^{p}}$, $\l \leqslant \left\|\psi_{a, \mu_i}\right\|^{2} \leqslant L$. 
Eq.~\eqref{eq:label3_6} follows from $f_{a, i}^{p} := \left\lceil\pi_{a, i} f^{p-1}\right\rceil$. 
Eq.~\eqref{eq:label3_7} follows from $\left\lceil\ x \right\rceil \geqslant x$. 
Eq.~\eqref{eq:label3_8} follows from the fact that summation is  the solution to the multi-agent G-optimal design, which equals $\sum_{a} d_{a}^{p} \leqslant d K$.  

We note that the upper bound also holds for $p = 1$ when $f^{0}=1$. This is because in the initialization step the matrix $V_a^0$ is defined as $V_{a}^{0} := \left( \sum_{j \in [M]} e_{a, j} e_{a, j}^{\top}\right)^{\dagger} \preceq \left( \sum_{j \in \R_{a}^{p}}\left\lceil\pi_{a, j} f^{0}\right] e_{a, j} e_{a, j}^{\top}\right)^{\dagger}$ for any $\pi$. This is true although the central server does not utilize the G-optimal design to obtain $\pi^{0}$ during initialization. Thus, Eq.~\eqref{eq:label3_6} still holds and this completes the proof.
\qed

\end{document}